\documentclass{article}

\usepackage{quoting}

\usepackage{hyperref}       
\usepackage{url}            
\usepackage{booktabs}       
\usepackage{amsfonts}       
\usepackage{nicefrac}       
\usepackage{microtype}      
\usepackage{amsthm}

\usepackage{amsmath,amssymb}
\usepackage{graphicx}

\newcommand{\cD}{\mathcal{D}}

\newcommand{\cG}{\mathcal{G}}
\newcommand{\cH}{\mathcal{H}}

\newcommand{\cX}{\mathcal{X}}
\newcommand{\cY}{\mathcal{Y}}

\usepackage[round]{natbib}
\usepackage[margin=4cm]{geometry}

\newcommand{\nats}{\mathbb{N}}
\newcommand{\half}{{\frac12}}

\newcommand{\err}{\mathrm{err}}
\newcommand{\norm}[1]{\|#1\|}
\DeclareMathOperator*{\argmin}{argmin}
\DeclareMathOperator*{\argmax}{argmax}
\newcommand{\dotprod}[1]{\langle #1 \rangle}
 \newcommand{\true}{\texttt{true}}
 \newcommand{\false}{\texttt{false}}

\usepackage{todonotes}

\usepackage{algorithm, algorithmicx}
\usepackage[noend]{algpseudocode}

\algnewcommand\algorithmicinput{\textbf{Input:}}
\algnewcommand\INPUT{\item[\algorithmicinput]}

\algnewcommand\algorithmicoutput{\textbf{Output:}}
\algnewcommand\OUTPUT{\item[\algorithmicoutput]}

\renewcommand{\P}{\mathbb{P}}

\def\E{{\mathbb E}}

\newcommand{\one}{\mathbb{I}}

\def\wh{\widehat}
\renewcommand{\eqref}[1]{Eq.~(\ref{eq:#1})}
\newcommand{\hx}{\wh{x}}
\newcommand{\unif}{\mathrm{Unif}}

\newtheorem{thm}{Theorem}
\newtheorem{definition}{Definition}
\newtheorem{lemma}[thm]{Lemma}
\newtheorem{cor}[thm]{Corollary}

\newcommand{\thmref}[1]{Theorem~\ref{thm:#1}}
\newcommand{\lemref}[1]{Lemma~\ref{lem:#1}}
\newcommand{\secref}[1]{Sec.~\ref{sec:#1}}

\newcommand{\appref}[1]{App.~\ref{app:#1}}
\newcommand{\myalgref}[1]{Alg.~\ref{alg:#1}}

\newcommand{\lab}{\texttt{label}}

\newcommand{\purealg}{\texttt{DFF18}}
\newcommand{\algknown}{\texttt{SR-DFF}}
\newcommand{\algunknown}{\texttt{PFR-DFF}}
\newcommand{\UB}{\mathrm{UB}}

\newcommand{\comment}[1]{$\triangle$ \emph{#1}}

\algnewcommand\algorithmicswitch{\textbf{switch}}
\algnewcommand\algorithmiccase{\textbf{case}}
\algnewcommand\algorithmicassert{\texttt{assert}}
\algnewcommand\Assert[1]{\State \algorithmicassert(#1)}

\algdef{SE}[SWITCH]{Switch}{EndSwitch}[1]{\algorithmicswitch\ #1\ \algorithmicdo}{\algorithmicend\ \algorithmicswitch}
\algdef{SE}[CASE]{Case}{EndCase}[1]{\algorithmiccase\ #1}{\algorithmicend\ \algorithmiccase}
\algtext*{EndSwitch}
\algtext*{EndCase}

\newcommand{\papertitle}{Improved Robust Algorithms for Learning with Discriminative Feature Feedback
}
\begin{document}

\title{\papertitle}

\author{Sivan Sabato}
\date{\small Department of Computer Science,\\
  Ben-Gurion University of the Negev, Israel}

\maketitle

\begin{abstract}
  \emph{Discriminative Feature Feedback} is a setting first introduced by \cite{DasguptaDeRoSa18}, which provides a protocol for interactive learning based on feature explanations that are provided by a human teacher. The features  distinguish between the labels of pairs of possibly similar instances. That work has shown that learning in this model can have considerable statistical and computational advantages over learning in standard label-based interactive learning models.

  In this work, we provide new robust interactive learning algorithms for the Discriminative Feature Feedback model, with mistake bounds that are significantly lower than those of previous robust algorithms for this setting. In the adversarial setting, we reduce the dependence on the number of protocol exceptions from quadratic to linear. In addition, we provide an algorithm for a slightly more restricted model, which obtains an even smaller mistake bound for large models with many exceptions.
  In the stochastic setting, we provide the first algorithm that converges to the exception rate with a polynomial sample complexity. Our algorithm and analysis for the stochastic setting involve a new construction that we call \emph{Feature Influence}, which may be of wider applicability.
\end{abstract}

\section{Introduction}\label{sec:intro}

Interactive Machine Learning is an umbrella term for settings in which the learning algorithm interactively obtains feedback from the environment while learning. Such settings, and especially Active Learning, have drawn much interest in recent years, as interactivity can significantly improve the results of learning \cite[see, e.g.,][]{Hanneke14}. Incorporating human feedback beyond labels into the interactive process has a great potential for further improvements, as it can allow learning with less data, and can result in more accurate models that are also more interpretable.

We consider \emph{Discriminative Feature Feedback}, an interactive learning model first introduced by \cite{DasguptaDeRoSa18}. This model provides a formal framework for the idea that human teachers can in many cases provide \emph{explanations} for the reasons why certain examples should be classified in a specific way, based on a comparison between examples. For instance, when classifying patients according to their underlying condition, an expert may explain that patient A does not have the same condition as patient B, even though their parameters may seem similar, because patient A has a blood test result that is inconsistent with the condition of patient B. \cite{DasguptaDeRoSa18} showed that this richer type of feedback can make learning significantly easier than standard label-based approaches. An important property of the Discriminative Feature Feedback model is that the a-priori number of possible features (the dimensionality of the problem) need not be bounded. Indeed, the provided mistake bounds are independent of the number of possible features. 

The model proposed in \cite{DasguptaDeRoSa18} was nonetheless limited, as it required the feedback from the human teacher to always conform to the Discriminative Feature Feedback protocol. Deviations from the protocol can occur because of human error, as well as because of inherent exceptional behavior of certain examples, such as patients who have a rare condition that is difficult to identify using available measures. \cite{DasguptaSa20} showed that handling such deviations effectively requires relaxing the protocol assumptions of the original model, and considered a setting in which a bounded amount (number or rate) of exceptions from the feature feedback protocol is allowed. Under this relaxation, they provided robust algorithms for the adversarial setting and for the stochastic setting with dimension-independent mistake bounds.
However, the mistake bounds derived in \cite{DasguptaSa20} depend very strongly on the exception threshold: In the adversarial setting, the mistake bound is quadratic in the number of exceptions, and in the stochastic setting, whenever the exception rate is non-zero, the mistake rate of the algorithm does not converge to the exception rate.

In this work, we provide new algorithms for the adversarial setting and for the stochastic setting, which obtain significantly improved mistake bounds. For the adversarial setting, we achieve a linear dependence on the number of exceptions, and in the stochastic case we show that it is possible to converge to the exception rate, with a polynomial sample complexity.
We also show that in the adversarial setting, a slightly more restricted feature feedback model leads to a significantly improved mistake bound. Lastly, we show that the algorithms can be made parameter-free while keeping the mistake bound the same up to logarithmic factors.

Our algorithm and analysis for the stochastic setting employ a new construction that we call \emph{Feature Influence}, which may be of wider applicability. This construction allows separating the cost of identifying candidate features from the cost of constructing a low-error classification model over these features, and provides a generic recipe for constructing learning algorithms for a wide range of interactive settings.

\section{Related work}
Several previous works have studied explanation-based interactive learning in application domains. \cite{teso2019explanatory} use an interactive protocol with local machine-generated explanations with user-based corrections, as a way to guide interactive learning. \cite{schramowski2020making} show how deep neural network learning can be improved by using interactive explanations. \cite{GuoDaAlMaCoKn22} studied the effect of explanation-based interactive learning on learning outcomes.

Other works specifically suggest using some form of a feature feedback to improve learning results. \citep{CD90} studied the use of weakly predictive features, and variations of this idea were studied in various applications such as text and vision~\citep{RMJ05,DMM08,S11a,MSCPY18,liang2020alice}. \cite{PD17,VisotskyAtCh19} studied formal versions of these approaches and their impact on the sample complexity. 

Interactive learning with user feedback that provides features distinguishing between instances have been used in computer vision \citep{BWBSWPB10,ZCK15}. The formal Discriminative Feature Feedback setting was first introduced in \cite{DasguptaDeRoSa18}, and the first robust algorithms for this setting were proposed in \cite{DasguptaSa20}. In the next section, we provide a formal definition of the setting and discuss relevant previous results.

\section{Setting and Preliminaries}
\label{sec:setting}

The formal definition of the Discriminative Feature Feedback model provided below follows \cite{DasguptaDeRoSa18} and \cite{DasguptaSa20}.
Let $\cX$ be the domain of examples, and let $\cY$ be a finite domain of labels. Let $c^*:\cX \rightarrow \cY$ be the unknown target concept mapping examples to labels. Let $\Phi$ be a set of binary features, where each feature $\phi \in \Phi$ is some predicate $\phi: \cX \rightarrow \{\true, \false\}$. $\Phi$  can be infinite, and is not necessarily known to the learner. $\Phi$ is assumed to be closed under negation.

The Discriminative Feature Feedback protocol is defined with respect to a \emph{component representation} which is unknown to the learner:  For an integer $m$, let $\cG = \{ G_1,\ldots,G_m\}$ be a set of $m$ subsets of $\cX$ (components) that cover the domain, $\cX = \cup_{i\in [m]}G_i$. For $x \in \cX$, let $G(x)$ be some $G \in \cG$ such that $x \in G$. It is assumed that each $G \in \cG$ has an associated label $\ell(G) \in \cY$, and for each $x \in \cX$, $c^*(x) = \ell(G(x))$. Note that components with the same label may overlap.

The central assumption of the Discriminative Feature Feedback model is that for any two $G_i,G_j$ such that $\ell(G_i) \neq \ell(G_j)$, there exists a \emph{discriminative feature} $\phi(G_i,G_j)$, such that $x$ satisfies $\phi(G_i,G_j)$ if $x \in G_i$, and $x$ does not satisfy $\phi(G_i,G_j)$ if $x \in G_j$. Note that there could be more than one potential discriminative feature between two components. We denote by $\phi(\cdot, \cdot)$  a specific choice of such features that satisfies $\forall i,j, \phi(G_j, G_i) = \neg \phi(G_i, G_j)$.

We first describe the feature feedback interactive learning protocol in the pure, noiseless model. Each iteration of the protocol is of the following form:
\begin{itemize}
\item A new instance $x_t$ arrives.
\item The learning algorithm provides a predicted label $\widehat{y}_t$, and an instance $\widehat{x}_t$ which was previously observed with that label. This instance serves as the explanation for the predicted  label: ``$x_t$ is assigned label $\hat{y}_t$ because $\hat{x}_t$ was assigned label $\hat{y}_t$''.
\item If the prediction is correct, no feedback is obtained from the teacher.
\item If the prediction is incorrect, the teacher provides the correct label $y_t = c^*(x)$, and the discriminative feature $\phi(G(x_t),G(\widehat{x}_t))$. This feature explains why $x_t$ should be labeled differently from $\widehat{x}_t$.
\end{itemize}
For example, if the instances represent animal species and the labels represent the type of each animal, then a discriminative feature $\phi$ between a dolphin and a halibut could be ``has gills'', explaining why a dolphin is not a fish.

\cite{DasguptaSa20} relaxed the above model to allow deviations from the protocol, and studied the adversarial setting and the stochastic setting. In the adversarial setting, the input stream is arbitrary, and it is assumed that there exists a representation of size $m$ such that the interaction with the teacher conforms to the protocol above, except for at most $k$ exceptions. Here, an exception can mean that the teacher provides a label that is inconsistent with the representation or a discriminative feature that is inconsistent with the representation, or both. In the stochastic setting, it is assumed that the input stream is drawn i.i.d.~from some distribution over $\cX$, and that there exists a representation of size $m$ such that the interaction with the teacher conforms to the protocol above except for a rate of at most $\epsilon$ of  exceptions. Note that in both cases, there can be more than one pair of representation size ($m$) and exception threshold ($k$ or $\epsilon$) that are consistent with the input; The bounds discussed below hold for all such pairs. 

We use standard $O(\cdot)$ notation to indicates implicit constant factors, and $\tilde{O}(\cdot)$  to indicate implicit logarithmic factors.\cite{DasguptaDeRoSa18} provided an algorithm (henceforth denoted \purealg) for the adversarial setting with  no exceptions, with a mistake bound of $O(m^2)$. It was shown in \cite{DasguptaSa20} that this is the best possible dependence on $m$. \cite{DasguptaSa20} further showed that describing a representation of size $m$ with $k$ exceptions using a different representation with zero exceptions may require $\Omega(|\Phi|)$ components.  Thus, to obtain mistake bounds that are independent of $|\Phi|$, robust algorithms that can handle exceptions directly are crucial. Recall that $\Phi$ can be infinite, thus bounds that depend on $|\Phi|$ are undesirable.

\cite{DasguptaSa20} provided robust algorithms for the adversarial setting and for the stochastic setting with mistake bounds of $O(m^2k+mk^2)$ and $\tilde{O}( m^2\epsilon + m^2/\sqrt{n})$ respectively, where $n$ is the length of the stochastic input stream. These mistake bounds incur a substantial penalty for each exception. In this work, we derive new algorithms that obtain significantly improved mistake bounds. In particular, in the adversarial setting we prove a mistake bound of $m(m-1)+mk$, giving a linear dependence on the number of exceptions instead of a quadratic dependence, and in the stochastic case we show that it is possible for the mistake rate to converge exactly to the exception rate with a polynomial sample complexity, unlike the previous bound that only converges to $m^2$ times the exception rate.

\section{Algorithms for the adversarial setting}
\label{sec:agnostic-adversarial}

In this section, we consider the adversarial setting. In this setting, the
input stream can be arbitrary, and it is parameterized only by the
representation size $m$, and by the number of exceptions from the
feature-feedback protocol with respect to this representation, which we assume
is upper bounded by an integer $k$.  In \secref{newrobust}, we derive a robust
algorithm for this setting with a mistake bound that avoids the quadratic dependence on $k$ of the algorithm
proposed in \cite{DasguptaSa20}.  The mistake bound is proved in
\secref{mistakebound}. In \secref{simpler}, we study a more restricted
representation model that assumes a unique label for each component, and show
that under this model it is possible to obtain a mistake bound of
$2m(m-1) + 6k$. This is a substantially lower dependence on the number of exceptions for large models. This result pinpoints the challenge of
improving the $mk$ term in the unrestricted representation model to a
specific property of this model: the sharing of labels
between components.

\subsection{An improved robust algorithm for the adversarial setting}
\label{sec:newrobust}
We present a new robust algorithm for the adversarial setting with an improved mistake bound. Surprisingly, unlike the algorithm proposed in \cite{DasguptaSa20}, this new algorithm is only slightly different from the non-robust \purealg.
\algknown\ (Simple Robust DFF) is listed in \myalgref{robust}. %It calls the procedure \handle,\todo{this procedure is obsolete and should be removed} given in \myalgref{incorrect}.
Like \purealg, \algknown\ keeps track of the following information:
%\begin{itemize}
  The first labeled example $(x_0,y_0)$; A decision list $L$, which is represented as a list of prediction rules; Each rule in $L$ has  a conjunction $C[x]$ of features from $\Phi$, where $x$ is some example from $\cX$ that represents the rule, and a label associated with the rule, denoted $\lab[x]$.
The conjunctions in the list are iteratively refined based on the feedback from the teacher. A rule is {\it created} if an example that matches none of the existing conjunctions appears. A rule is {\it refined} if mistakes occur on an example that matched the rule. A rule is {\it deleted} if it becomes too long.  As we show below, a rule can become too long if it was refined using feedback based on exceptions, or if the rule itself was created based on an exception. The option to delete a rule is the main change compared to \purealg.

  We derive the following mistake bound for \algknown.
\begin{thm}\label{thm:mistakebound}
Suppose that the input stream is consistent with some representation of size $m$ with at most $k$ exceptions. Then \algknown\ makes at most $m(m-1) + mk$ mistakes. 
\end{thm}
This mistake bound is a factor of $k$ smaller compared to the mistake bound of $O(m^2k+mk^2)$ of the robust algorithm of \cite{DasguptaSa20}. 
Unlike the non-robust algorithm \purealg, \algknown\ requires the maximal number of components $m$ as input. The robust algorithm of \cite{DasguptaSa20} required both $m$ and $k$ as input, thus this is an improvement for robust algorithms. Moreover, the following theorem shows that the requirement to have  $m$ as input can also be relaxed, while keeping the order of the mistake bound the same up to logarithmic factors.
\begin{thm}\label{thm:unknownbound}
  Suppose that the input stream is consistent with some representation of size $m$ with at most $k$ exceptions. Let $\UB(m,k) := m(m-1) + mk$ be the mistake upper bound proved in \thmref{mistakebound}. There exists a parameter-free robust algorithm with a mistake bound of 
  \[
    32\UB(m,k)\log^2(8\UB(m,k)) = \tilde{O}(m^2+mk).
  \]
\end{thm}
This theorem is proved in \appref{parameterfree}. We now turn to prove the main mistake bound, \thmref{mistakebound}.

\begin{algorithm}[t]
  \caption{\algknown: A simple robust Discriminative Feature Feedback algorithm for the adversarial setting}
  \label{alg:robust}
  \begin{algorithmic}[1]
    \INPUT{Maximal number of components $m$.}
    \State $t \leftarrow 0$
    \State Get the label $y_0$ of the first example $x_0$.
    \State Initialize $L$ to an empty list.
    
    \While{\true}
    \State $t \leftarrow t+1$
    \State Get next example $x_t$.
      \If{$\exists C[\wh{x}] \in L$ such that $x_t$ satisfies $C[\wh{x}]$}
         \State Predict $\lab[\wh{x}]$ and provide explanation $\wh{x}$.
         \If{prediction is incorrect}
         \State Get correct label $y_t$ and feature $\phi$.
         \State $C[\wh{x}] := C[\wh{x}] \wedge \neg \phi$ \label{line:refine}
         \If{$|C[\wh{x}]| \geq m$}
             \State Delete $C[\wh{x}]$ from $L$.
             \EndIf
         \EndIf
      \Else ~ (no relevant rule exists)
         \State Predict $y_0$ and provide explanation $x_0$.
         \If{prediction is incorrect}
            \State Get correct label $y_t$ and feature $\phi$. 
            \State Add to $L$ an empty conjunction $C[x_t]$, and set $\lab[x_t] \leftarrow y_t$.
         \EndIf
       \EndIf
       \EndWhile
     \end{algorithmic}
     
   \end{algorithm}

\subsection{Bounding the mistakes of \algknown}\label{sec:analysis}\label{sec:mistakebound}

We now prove \thmref{mistakebound} that bounds the number of mistakes made by \algknown. Assume some representation $\cG$ of size $m$ such that the input has at most $k$ exceptions with respect to this representation. 
We call a rule maintained by the algorithm \emph{corrupted} if at least one of the features in its conjunction was added in line \ref{line:refine} when the example $x_t$ was an exception. We say that a rule $C[x]$ was created by an exception if $x$ is an exception.  Note that a rule can be created by an exception and still not be considered corrupted.
Call non-exception examples \emph{valid examples}. Call rules created by valid examples \emph{valid rules}.

We start by proving several invariants of the algorithm. First, we prove that valid non-corrupted rules are always satisfied by a whole component in the representation $\cG$. 

\begin{lemma}\label{lem:gx}
  At all times in the algorithm, if $C[\hx]$ is a valid and non-corrupted rule, then:
  \begin{itemize}
  \item Conjunction $C[\hx]$ is satisfied by every example in $G(\hx)$.
  \item For every feature $\phi$ in $C[\hx]$, there is some valid example $x \in \cX$ such that $\phi(G(\hx),G(x)) = \phi$.
  \end{itemize}
\end{lemma}

\begin{proof}
  We prove the claim by induction on the length of $C[\hx]$.
  When $C[\hx]$ is first created, it is an empty conjunction so it is satisfied by all of $G(\hx)$.  When $C[\hx]$ is refined by $\neg \phi$, since the rule is not corrupted, the example $x$ for which $\phi$ was provided is valid, hence $\neg\phi = \phi(G(\hx), G(x))$. This implies that $G(\hx)$ has no examples that are satisfied by $\phi$. Hence, after adding $\neg \phi$ to $C[\hx]$, the refined $C[\hx]$ is still satisfied by $G(\hx)$ and is separated by $\phi$ from $G(x)$.
 \end{proof}

 Next, we prove that two non-corrupted valid rules have representatives from different components.
 \begin{lemma}\label{lem:dup}
   For any two valid non-corrupted rules $C[x]$ and $C[x']$ that exist at the same time in $L$, $G(x) \neq G(x')$. 
 \end{lemma}
 \begin{proof}
   Suppose $x$ was observed earlier in the input sequence and $x'$ was observed later; Since $C[x]$ and $C[x']$ were both created, this means that $C[x]$, in its form when $x'$ was observed, was not satisfied by $x'$. But by \lemref{gx}, $C[x]$ is always satisfied by $G(x)$. Hence, $x' \notin G(x)$, which implies the claim.
 \end{proof}

 Next, we prove that only rules affected by exceptions might be deleted.
 \begin{lemma}\label{lem:delete}
   If the rule $C[\hx]$ is deleted then at least one of the two holds: $\hx$ is an exception or $C[\hx]$ is corrupted. 
 \end{lemma}
 \begin{proof} 
   Assume for contradiction that $\hx$ is not an exception and $C[\hx]$ is not corrupted, but rule $C[\hx]$ is deleted. Thus, the conjunction $C[\hx]$ has at least $m$ features.  By \lemref{gx}, for each feature in $C[\hx]$ there is some valid example $x$ such that $G(x)$ is separated from $G(\hx)$ using that feature. Moreover, after $C[\hx]$ is refined with this feature, no additional examples from $G(x)$ will be satisfied by this rule. Thus, each component can contribute at most one feature to $C[\hx]$. Since there are $m$ components in $\cG$, there are at most $m-1$ features in $C[\hx]$, a contradiction to the size of $C[\hx]$. 
 \end{proof}

We bound the total number of mistakes by first bounding the total number of rules created by the algorithm.
\begin{lemma}
 \algknown\ creates at most $m+k$ rules. \label{lem:rgen}
\end{lemma}
\begin{proof}

  By \lemref{dup}, at any time in the algorithm, the total number of valid non-corrupted rules is at most the number of components, $m$. Moreover, by \lemref{delete}, such rules are never deleted. Therefore, the total number of such rules that are generated by the algorithm is at most $m$. 
  In addition, any one exception cannot both generate a new rule and corrupt an existing one. Therefore, the total number of corrupted rules and rules generated by exceptions over the entire run of the algorithm is at most $k$. Thus, the total number of rules created by the algorithm is at most $m+k$. 
\end{proof}
The proof of \thmref{mistakebound} is now immediate, as follows: 
Any rule that makes $m$ mistakes is of length $m$ and is thus deleted. Therefore, each rule makes at most $m-1$ mistakes, except rules that end up being deleted --- these make  $m$ mistakes. By \lemref{rgen}, at most $m+k$  rules are created by the algorithm, and by \lemref{delete}, at most $k$ of them are deleted. Therefore, the total number of mistakes made by the algorithm is at most $m(m-1) + mk$. This completes the proof of \thmref{mistakebound}.

\subsection{An improved mistake bound for a restricted model}
\label{sec:simpler}
According to the mistake bound in \thmref{mistakebound}, every additional exception can result in as many as $m$ additional mistakes. Is there an algorithm that pays only a constant number of mistakes for every additional exception in the adversarial setting? We now show that such an algorithm does exist for a more restricted representation model. It remains open whether a similar result can be obtained for the general model.

In the restricted representation model that we now consider, each component has a unique label. We prove the following mistake bound.

\begin{thm}\label{thm:mistakeboundsimple}
  Suppose that the input stream is consistent with some representation of size $m$ with at most $k$ exceptions. Suppose further that each component has a unique label. Then there exists a (randomized) algorithm with an expected number of mistakes of at most $2m(m-1) + 6k$. 
\end{thm}

The uniqueness of the label is used by the algorithm to identify which rule needs correction in the case of a mistake. This pinpoints the challenge for the general representation model, and may help in the future to obtain a similar result for that model, or prove its impossibility.
The algorithm is listed in \myalgref{robustsimple}. It accepts parameters $p \in [0,1]$ and $l \in \nats$. \thmref{mistakeboundsimple} is proved using $l = m-1$ and $p = 1/(m-1)$.

Similarly to \algknown, \myalgref{robustsimple} refines rules based on feedback on incorrect labels. In addition to the list of rules, \myalgref{robustsimple} also maintains a counter $U$ for each rule, of the number of updates that have been applied to the rule.\footnote{For simplicity, we include in the count of updates also cases where an updated was avoided due to inconsistency (line \ref{line:validphi} in  \myalgref{robustsimple}) and term these also ``updates'' below.} When $U$ becomes too large, the rule is deleted. An important difference from \algknown\ is in the behavior of the algorithm when no rule is consistent with the example. In this case, when getting the true label of the example, if no rule with this label already exists, the algorithm creates one with an independent  probability of $p$. If such a rule does exist, then one of the features in its conjunction that disagree with the current example is removed; Note that such a feature must exist, otherwise the rule would be consistent with the example.

\begin{algorithm}[h]
  \caption{Robust Discriminative Feature Feedback for the adversarial setting under the unique-label assumption}
  \label{alg:robustsimple}
  \begin{algorithmic}[1]

\INPUT{$p \in [0,1]$, $l \in \nats$}
\State $t \leftarrow 0$
\State Get the label $y_0$ of the first example $x_0$.
\State Initialize $L$ to an empty list.
\While{\true}
\State $t \leftarrow t+1$
    \State Get next example $x_t$.
\If{$\exists C[\wh{x}] \in L$ such that $x_t$ satisfies $C[\wh{x}]$}
\State Predict $\lab[\wh{x}]$ and provide example $\wh{x}$.
\If{prediction is incorrect}
 \State Get correct label $y_t$ and feature $\phi_t$.
\If{$\hat{x}$ satisfies $\neg\phi_t$ and $x_t$ satisfies $\phi_t$}\label{line:validphi}
\State Add $\neg \phi_t$ to the conjunction $C[\hat{x}]$.\label{line:added} 
\EndIf

\State $U(\hx) \leftarrow U(\hx) + 1$.
 \State \textbf{if} $U(\hx) \geq m+l-1$ \textbf{then} delete rule $C[\hx]$. 
\EndIf
         \Else ~~~\comment{No rule is satisfied by $x_t$}
\State Predict $y_0$ and provide example $x_0$.
\If{prediction is incorrect}
 \State Get correct label $y_t$ and feature $\phi_t$.
              \If{there is no rule in $L$ with label $y_t$}
                   \State Draw an independent bit $B$ with success probability $p$.
\If{B = 1}
\State Add to $L$ a new empty rule $C[x_t]$ with $\lab[x_t]=y_t$;
\State $U(x_t) \leftarrow 0$.
\EndIf
              \Else 
                   \State Let $C[\hx]$ be the rule in $L$ with label $y_t$.
                   \State Remove from conjunction $C[\hx]$ some feature that is not satisfied by $x_t$.\label{line:delete}

                   \State $U(\hx) \leftarrow U(\hx) + 1$.
                   \State \textbf{if} $U(\hx) \geq m+l-1$ \textbf{then} delete rule $C[\hx]$. 
              \EndIf
         \EndIf
\EndIf
\EndWhile
\end{algorithmic}
\end{algorithm}

The deletion mechanism allows the algorithm to correct corrupted rules using subsequent examples. Several key ideas are incorporated into the algorithm, and lead to a mistake bound with only a constant factor over $k$:
\begin{itemize}
\item A low probability of creating a new rule ensures that exceptions do not cause the creation of too many rules.
\item Removing single features from rules instead of deleting the entire rule limits the effect of exceptions on the list of rules. This is possible due to the unique-label assumption.
\item Rules are deleted only if they are quite long, so that each corrupted rule can absorb several exceptions.
  \end{itemize}
Overall, no single mistake can cause too large a change to the list of rules; This stability property keeps the dependence of the bound on the number of exceptions low.

We now prove \thmref{mistakeboundsimple}. First, note that due to the unique-label assumption, for any valid rule $C[\hx]$ with label $y$, any valid $x_t$ with the same label must be in $G(\hx)$. Therefore, every time a literal is deleted from a rule in line \ref{line:delete}, if $x_t$ and $\hx$ are both valid examples, removing a literal that is inconsistent with $x_t$ from $C[\hx]$ is necessarily a correct deletion, in that it removes a feature that does not separate $G(\hx)$ from any other component.

Consider a specific run of the algorithm. To upper bound the total number of mistakes, we distinguish the types of mistakes:
\begin{itemize}
\item A mistake on an exception $x_t$; There are at most $k$ such mistakes.
\item A mistake on a valid example:
  \begin{itemize}
  \item A mistake that caused an update (of any kind) in a rule created by an exception (an exception rule). Let $M_{ev}$ denote the number of mistakes in exception rules by valid examples.
  \item A mistake that caused the refinement of a rule by adding a literal. Let $A_{vv}$ denote the number of additions of literals to valid rules by valid examples.
     \item A mistake that caused an attempt to create a rule (that is, a random draw of a bit $B$). Denote the number of attempts to create a rule by a valid example by $\tilde{R}_v$. 
  \item A mistake that caused a deletion of a literal from a valid rule.  Let $D_{vv}$ denote the number of deletions of literals from valid rules by valid examples.
    \end{itemize}
  \end{itemize}
  The total number of mistakes $M_T$ made by the algorithm is thus upper bounded by:
\[
  M_T \leq k+ M_{ev} + A_{vv} + \tilde{R}_v + D_{vv}.
\]
Let $R_v$ be the number of rules that were actually created based on a valid examples (cases where $B = 1$). We have $\E[R_v] = p\E[\tilde{R}_v]$, thus, rearranging, we have
\begin{equation} \label{eq:mbounda}
  \E[M_T] \leq  k + \E[R_v]/p + \E[A_{vv}] + \E[D_{vv}] + \E[M_{ev}].
\end{equation}

Next, we bound $A_{vv}$ and $D_{vv}$. Each valid rule has at most $m-1$ true discriminative features. However, true features might be deleted from a valid rule by an exception and could then be re-added. Denote by $D_{ve}$ the total number of true discriminative features deleted from valid rules by exceptions. It follows that $A_{vv} \leq (m-1)R_v+D_{ve}$.
To bound $D_{vv}$, note that every incorrect literal in a valid rule deleted by a valid example must have been added by an exception. Denote by $A_{ve}$ the total number of incorrect literals added to some valid rule by an exception. Then $D_{vv} \leq A_{ve}$. 
It follows that 
  \begin{equation}\label{eq:mboundb}
    \E[M_T] \leq k+\E[R_v](1/p+ (m-1)) + \E[D_{ve}]+ \E[A_{ve}]+ \E[M_{ev}] .
  \end{equation}

Next, we upper bound $R_v$. The algorithm only creates a rule with label $y_t$ if such a rule does not currently exist. There are at most $m$ valid labels, thus the total number of valid rules created during the run is at most $m$ plus the number of deleted valid rules. To bound the number of deleted valid rules, observe that a rule is only deleted if at least $m+l-1$ updates, as counted by the counter $U$, have been applied to it. The following lemma links the number of rule updates to the number of exceptions that caused updates to the rule. Since the number of exceptions is bounded, this will lead to a bound on the number of deleted rules.

    \begin{lemma}\label{lem:updates}
      Let $t$ be some iteration of the run of \myalgref{robustsimple}. Let $C[x]$ be some valid rule that exists during this iteration, and let $U(x)$ be the value of the update counter for this rule at the end of this iteration. The number of exceptions that caused updates to this rule until the end of this iteration is at least $(U(x)-(m-1))/2$. 
    \end{lemma}
    \begin{proof}
      Let $a$ be the total number of literals that were added to $C[x]$ until the end of iteration $t$, and let $b$ be the number of features that were deleted. Let $c$ be the number of examples that were ignored due to failing the condition on line \ref{line:validphi}. Note that $U(x) = a+b + c$. All of the ignored examples must have been exceptions, thus at least $c$ updates were caused by exceptions.
      
      For every deletion of a literal from a valid rule, at least one of its addition and its deletion must have been caused by an exception, since they cannot both be correct. Therefore, the number of exceptions that have led to additions or deletions in this rule so far is at least $b$. In addition, out of the $a$ literals that were added to the rule, at most $m-1$ can be correct. Therefore, at least $a-(m-1)$ of the existing features were caused by exceptions.

      It follows that the total number of exceptions that caused updates to this rule is at least
      \begin{align*}
        &c+\max(a-(m-1),b) \geq (a+b+c-(m-1))/2 \geq (U(x)-(m-1))/2.
      \end{align*}
      This completes the proof.
    \end{proof}

To bound $R_v$, note that when a rule $C[x]$ is deleted, it satisfies $U(x) = m-1+l$. Therefore, if the rule is valid, it follows from \lemref{updates} that it has been updated by at least $l/2$ exceptions.
Let $U_{ve}$ be the number of updates to valid rules by exceptions. Then the total number of deletions of valid rules is at most $U_{ve}/(l/2)$, hence $R_v \leq m + 2U_{ve}/l$.
In addition, $D_{ve}+A_{ve}\leq U_{ve}$. Thus, from \eqref{mboundb}, 
    \begin{align}
      \E[M_T] &\leq k+((m + 2U_{ve}/l)(1/p + m-1)) + \E[U_{ve}]+ \E[M_{ev}]\notag\\
      &= k+m/p+m^2-m+(2(1/p+m-1)/l+1)\E[U_{ve}]+\E[M_{ev}].\label{eq:mev}
    \end{align}
    
      Next, we upper bound $M_{ev}$. Let $R_e$ be the number of created exception rules. An exception rule is deleted if it accumulates $m+l-1$ updates, hence $M_{ev} \leq R_e(m+l-1)$. 
   In addition, letting $\tilde{R}_e$ be the number of attempts to create an exception rules (draws of $B$), we have $U_{ve}+\tilde{R}_e \leq k$. 
  In addition, $\E[R_e] = p\E[\tilde{R}_e]$. Therefore,
  \[
    \E[U_{ve}] \leq k - \E[\tilde{R}_e] = k - \E[R_e]/p.
  \]
  Combining these bounds with \eqref{mev}, we get
       \begin{align*}
         \E[M_T] &\leq k+m/p+m^2-m + (2(1/p+m-1)/l+1)(k-\E[R_e]/p) +\E[R_e](m+l-1)\\
                 &= m/p+m^2-m + (2(1/p+m-1)/l+2)k \\
         &\quad+ (m+l-1-\frac{1}{p}(2(1/p+m-1)/l+1))\E[R_e].
       \end{align*}
       Setting $l = 1/p$, the coefficient of $\E[R_e]$ becomes
       \[
         m+1/p-1-2(m+1/p-1)-1/p = -m - 2/p +1< 0,
       \]
       and the coefficient of $k$ is $2(1/p+m-1)/l + 2 = 4+2p(m-1).$
      
       Setting $p = 1/(m-1)$, it follows that $\E[M_T] \leq 2m(m-1) + 6k.$

This completes the proof of \thmref{mistakeboundsimple}.

\section{The stochastic setting --- approaching the optimal mistake rate}\label{sec:stochastic}
\newcommand{\comp}{\mathrm{comp}}
\newcommand{\mycomp}{{\mathrm{comp}+}}

In this section, we assume that the input stream of examples is drawn from a stochastic source, which outputs i.i.d.~examples according to a distribution. It is assumed that the distribution and the provided teacher feedback are consistent with some representation $\cG$ of size $m$ and its induced concept $c^*$, except for a rate of up to $\epsilon \in [0,1)$ of exceptions.

Formally, there exists some marginal distribution $\cD_X$ over examples from $\cX$. Examples are drawn i.i.d.~according to $\cD_X$. For each draw of $X \sim \cD_X$, there is some probability that an exception occurs, which can depend on the value of $X$. The overall probability of an exception is at most $\epsilon$. If an exception does not occur, then the label feedback and feature feedback, if provided, are consistent with the representation and the feedback protocol. 
Letting $\cY$ be the set of possible labels, we denote by $\cD$ the distribution over $\cX \times \cY$ which is induced by drawing an example $X \sim \cD_X$ and assigning the label $Y = c^*(X)$ if there is no exception, and some example and label induced by the exception otherwise.\footnote{In fact, in case of an exceptions the example and label can be adversary and do not need to be random. We adhere to the distribution formulation for simplicity of presentation.} The labeled examples that are provided to the learner in the stochastic setting are distributed as i.i.d.~draws from $\cD$.

The algorithms for the adversarial settings provided in \secref{agnostic-adversarial}, as well as their mistake bounds,  apply also in a stochastic setting. However, since $\E[k] = \epsilon n$, \thmref{mistakebound} gives a mistake rate of $m(m-1)/n+\epsilon m$, so that even when $n \rightarrow \infty$, the mistake rate of \algknown\ never converges to the exception rate, but to a factor of $m$ over it. The bound  for the stochastic setting in \cite{DasguptaSa20} gives an even larger factor of $m^2$. \thmref{mistakeboundsimple} when applied to the stochastic setting shows that it is possible to converge to the true mistake rate up to a constant factor, but only for the unique-label assumption.
In this section, we show that in the stochastic setting with the general representation model, it is possible for the mistake rate to converge to the exception rate with no constant factors, and with a polynomial sample complexity. The proposed algorithm is not efficient; An open question for future work is whether this can be achieved with an efficient algorithm.

To prove the result, we propose a new general notion of \emph{Feature Influence}, which may be of wider applicability. Consider some hypothesis class $\cH_\Phi \subseteq \cY^\cX$ that is defined over a set of features $\Phi$, such that for any subset of the features $\Phi' \subseteq \Phi$ we have $\cH_{\Phi'} \subseteq \cH_{\Phi}$.
Assume some feature feedback protocol. A \emph{feature discovery protocol}, which is an algorithm that interacts with the environment and the teacher, as specified in the feature feedback protocol, and outputs some feature subset $\hat{\Phi} \subseteq \Phi$. Feature Influence, defined below, measures the effect on the distribution error of restricting the hypothesis class to features from $\hat{\Phi}$.

Formally, given a distribution $\cD$, the error of a prediction rule $h:\cX \rightarrow \cY$ is $\err(h,\cD) := \P_{(X,Y) \sim \cD}[h(X) \neq Y]$. For a sequence of labeled examples $S = ((x_1,y_1),\ldots,(x_n,y_n))$, the empirical error is $\err(h,S) := \err(h,\unif(S))$, where $\unif(S)$ is the uniform distribution over $S$. Given a hypothesis class $\cH \subseteq \cY^\cX$, the smallest error that can be obtained by some hypothesis in $\cH$ on a distribution $\cD$ is denoted by $\err(\cH,\cD) = \inf_{h \in \cH} \err(h,\cD)$.

\begin{definition}[Feature Influence]
  Fix an exception rate $\epsilon \in [0,1]$, an excess error target $\alpha \in [0,1]$, a confidence level $\delta \in (0,1)$, a mistake bound $b \in \nats$, and a capacity bound $d \in \nats$. Assume some feature feedback protocol with some class $\cH := \cH_\Phi$. The setting has $\alpha$-\emph{Feature Influence} with exception rate $\epsilon$, confidence $\delta$, mistake bound $b$ and capacity bound $d$, if there exists a feature discovery protocol that outputs a feature subset $\hat{\Phi}$ such that for any stochastic i.i.d.~stream which is consistent with the feature feedback protocol except for an exception rate of up to $\epsilon$, if the protocol is applied to a prefix of an i.i.d.~sample drawn according to the distribution $\cD$ over $\cX \times \cY$ induced by the setting and stops before making more than $b$ mistakes, then:
  \begin{itemize}
    \item With probability at least $1-\delta$, $\err(\cH_{\hat{\Phi}},\cD) \leq \epsilon+\alpha$,
    \item $\P[\log_2|\cH_{\hat{\Phi}}| \leq d] = 1$.
    \end{itemize}
  \end{definition}

The following theorem provides a sample complexity bound for the stochastic setting for a general feature-based protocol and hypothesis class, using the notion of Feature Influence. 

\begin{thm}\label{thm:inf}
Suppose that a feature feedback setting with some $\cH_\Phi$ has $\alpha$-\emph{Feature Influence} with exception rate $\epsilon$, confidence $\delta$, mistake bound $b$ and capacity bound $d$. 

  Then there exists an algorithm (not necessarily efficient) for the stochastic setting that obtains a mistake rate of at most $\epsilon + 2\alpha$ with a probability at least $1-3\delta$ over i.i.d.~streams of length at least 
  \[
    \tilde{O}\left(\frac{b}{\alpha} + \left(\frac{\epsilon}{\alpha^3}+\frac{1}{\alpha^2}\right)(d+\log(1/\delta))\right).
  \]

\end{thm}

\begin{proof}
  Denote the stream length by $n$. Consider an algorithm that works in three stages:
  
  \begin{enumerate}
\item Run the feature discovery protocol  that witnesses the Feature Influence with mistake limit $b$ (thus observing some $n_1$ examples),and obtain the feature subset  $\hat{\Phi}$. Set $\hat{\cH} := \cH_{\hat{\Phi}}$. \label{stageinter}
\item Observe the next $n_2 := \alpha n/2$ examples, predicting an arbitrary label for each example, and observing the true label provided by the teacher. Let $S_2$ be the sequence of $n_2$ labeled examples obtained this way. Find some $\hat{h} \in \argmin_{h \in \hat{\cH}}\err(h,S_2)$. 

  \item Use $\hat{h}$ to predict the labels for the rest of the input stream.
  \end{enumerate}

Let $\cD$ over $\cX \times \cY$ be the distribution of labeled examples from which the i.i.d.~stream is drawn.
  The expected mistake rate during Stage 3 of the algorithm is $\err(\hat{h},\cD)$.   
  The length of Stage 3 is $n_3 = n - n_1 - n_2$.By Hoeffding's inequality, with probability at least $1-\delta$,
  the total number of mistakes of $\hat{h}$ during Stage 3 is at most $n_3\cdot \err(\hat{h},\cD) + \sqrt{n_3\log(1/\delta)/2}$. Since $n_3/n < 1$, the  mistake rate over all the stages is at most $(b+n_2)/n + \err(\hat{h},\cD)+ \sqrt{\log(1/\delta)/(2n)}$.

  Since $\hat{h}$ is the empirical risk minimizer of $\hat{\cH}$ over the i.i.d.~sample $S \sim \cD^{n_2}$,  we have,  for the case of binary classification ($|\cY| = 2$) \citep{BoucheronBoLu05}, that with a probability of at least $1-\delta$ over the randomness of the input sample during Stage 2,
 \[
    \err(\hat{h},\cD) \leq \err(\hat{\cH},\cD) +\tilde{O}\left(\sqrt{\err(\hat{\cH},\cD) \frac{d + \log(1/\delta)}{n_2}} + \frac{d + \log(1/\delta)}{n_2}\right).
  \]
  A standard reduction argument from multiclass learning to binary classification \citep{DanielySaBeSh15} shows that the inequality above holds also for a general finite $\cY$.

From the definition of Feature Influence, we have that with a probability at least $1-\delta$, $\err(\hat{\cH}, \cD) \leq \epsilon + \alpha.$
  Therefore, with a probability at least $1-3\delta$, the total mistake rate of the algorithm over all the stages is at most
  \begin{align*}
    &\frac{b+n_2}{n} + \epsilon + \alpha + \tilde{O}\left(\sqrt{\frac{(\epsilon + \alpha)(d+\log(1/\delta))}{n_2}}+ \frac{d + \log(1/\delta)}{n_2}\right)+ \sqrt{\frac{\log(1/\delta)}{2n}}=\\
    &\epsilon + \frac32\alpha + \frac{b}{n} + \tilde{O}\left(\sqrt{\left(\frac{\epsilon}{\alpha}+1\right)\frac{ (d+\log(1/\delta))}{n}}+\frac{d + \log(1/\delta)}{\alpha n}\right).
      \end{align*}
      To complete the proof, note that a stream length of
      \[
      n = \tilde{O}\left(\frac{b}{\alpha} + \left(\frac{\epsilon}{\alpha^3}+\frac{1}{\alpha^2}\right)(d+\log(1/\delta))\right)
       \]
      suffices for the expression above to be $\epsilon + 2\alpha$. 
\end{proof}

To apply this general result to our Discriminative Feature Feedback setting, we now define an appropriate hypothesis class $\cH_\Phi$ and prove a Feature Influence property.
Let $\cH^\comp_\Phi \subseteq \cY^\cX$ be the set of all concepts that are consistent with some component representation of size at most $m$ with features from $\Phi$, as defined in \secref{setting}. We have $\err(\cH_\Phi^{\mathrm{comp}}, \cD) \leq \epsilon$. 
To implement the feature discovery protocol, we define an additional hypothesis class $\cH^+_\Phi \supseteq \cH^\comp_\Phi$ that is more expressive than our component model. Let $\cH^+_\Phi$ be the class of concepts that can be described using a decision list with at most $m$ rules, where each rule is a conjunction of up to $m-1$ features with an associated label. Examples can be satisfied by more than one rule. The first rule in the decision list that is satisfied determines the label of the example. 
It is easy to see that all the possible true concepts $c^*$ can be represented using a hypothesis from $\cH^+_\Phi$. Hence, $\cH^+_\Phi \supseteq \cH^\comp_\Phi$ and $\err(\cH^+_\Phi,\cD) \leq \err(\cH^\comp_\Phi,\cD) \leq \epsilon$.

\newcommand{\ex}{\texttt{Ex}}
To show a Feature Influence property, we first show that if $\cH^+_\Phi$ is restricted to a subset of the features in an appropriate way, then the additional incurred error can be bounded. Let $\ex$ be the event than an exception occurred in the current draw of $X \sim \cD_X$. 
Denote $P[G]:= \P[X \in G \wedge \neg \ex]$.

Recall that as defined in \secref{setting}, $\Phi$ is closed under negation. We assume for simplicity that in every pair of features $\phi,\phi'$ such that $\phi = \neg \phi'$, one of the two is designated as the positive feature while the other is the negative feature. 
Denote the positive feature that separates two components $G$ and $G'$ with different labels by $\phi_+(G,G')$. Note that $\phi_+(G,G') = \phi_+(G',G)$. Let $P_\phi$ be the set of all sets of size 2 of components in $\cG$ such that the two components are separated by $\phi$ in the representation:
\[
  P_\phi := \{ \{G,G'\} \mid \phi_+(G,G') = \phi\}.
\]
Denote
    \[
      \beta_\phi := \sum_{\{G,G'\}\in P_\phi} P[G]P[G'],
    \]
    and let $\Phi_\beta := \{\phi \mid \beta_\phi \geq \beta\}$.
We first upper bound the effect of removing from $\Phi$ only features with a small $\beta_\phi$.
  \begin{lemma}\label{lem:errhb}
    $\err(\cH^+_{\Phi_\beta}, \cD) \leq \epsilon + \sqrt{\beta}\cdot m^2/2$.
  \end{lemma}
  \begin{proof}
    Consider the hypothesis $h_\beta \in \cH^+_{\Phi_\beta}$, defined as a decision list with a rule for each component $G \in \cG$, associated with the label $\ell(G)$. The rule is a conjunction of all the discriminative features separating $G$ from other components, except for those features that have $\beta_\phi < \beta$. The rules in the decision list defining $h_\beta$ are in descending order based on the marginal probability $P[G]$ for each $G \in \cG$, so that when more than one rule is satisfied by an example, the label of the component with the largest marginal probability is predicted.

    To bound the error of $h_\beta$, consider an example $x$ in some component $G$. If $x$ is predicted by $h_\beta$ with a different label than $c^*(x)$, then it satisfies the rule of another component $G'$ which has a larger (or equal) marginal probability than $G$ and a different label. Moreover, the discriminative feature $\phi$ separating $G$ and $G'$ must satisfy $\beta_\phi < \beta$, or it would have been a part of the rule of $G'$ and $x$ would not have satisfied it. It follows that
    \begin{align*}
      \P[(X \in G \wedge \neg \ex) \wedge (h_\beta(X) \neq c^*(X))] &\leq
                                                                      P[G] \cdot\one[\exists G' \in \cG, \beta_{\phi_+(G,G')}\leq \beta \wedge P[G'] \geq P[G]]\\
    &\leq P[G]\cdot\sum_{G': \beta_{\phi_+(G,G')}\leq \beta}  \one[P[G'] \geq P[G]].
    \end{align*}
    Summing over all components, we have that 
      \begin{align*}
        \P[\neg \ex \wedge (h_\beta(X) \neq c^*(X))] &\leq \sum_{\{G,G'\}\subseteq \Phi:\beta_{\phi_+(G,G')}\leq \beta}\min\{P[G'],P[G]\}\\
        & = \sum_{\phi: \beta_\phi \leq \beta}\sum_{\{G,G'\}\in P_\phi}\min\{P[G'],P[G]\}.
      \end{align*}
      For any given $\phi$, we have
      \[
        \sum_{\{G,G'\}\in P_\phi}\min\{P[G'],P[G]\} \leq \sum_{\{G,G'\}\in P_\phi}\sqrt{P[G']P[G]} \leq \sqrt{|P_\phi|\sum_{\{G,G'\}\in P_\phi}P[G']P[G]}.
      \]
      Let ${\boldsymbol 1}$ be the all-1 vector of dimension $|P_\Phi|$, and let ${\boldsymbol v}$ be the vector of the same dimension with coordinate values $\sqrt{P[G']P[G]}$. By the Cauchy-Schwartz inequality,
      \[
        \sum_{\{G,G'\}\in P_\phi}\sqrt{P[G']P[G]} = \dotprod{{\boldsymbol 1},\boldsymbol v} \leq \norm{\boldsymbol 1}_2 \norm{\boldsymbol v}_2 = \sqrt{|P_\phi|\sum_{\{G,G'\}\in P_\phi}P[G']P[G]}.
        \]

      From the definition of $\beta_\phi$, it follows that
      \[
        \sum_{\{G,G'\}\in P_\phi}\min\{P[G'],P[G]\} \leq \sqrt{|P_\phi|\beta_\phi}.
      \]
      Therefore,
      \[
      \P[\neg \ex \wedge (h_\beta(X) \neq c^*(X))] \leq \sum_{\phi: \beta_\phi \leq \beta}\sqrt{|P_\phi|\beta_\phi} \leq \sqrt{\beta}\sum_{\phi: \beta_\phi \leq \beta}\sqrt{|P_\phi|} \leq  \sqrt{\beta} \cdot m^2/2.
    \]
    The last inequality follows since $\sum_{\phi: \beta_\phi \leq \beta}\sqrt{|P_\phi|} \leq \sum_{\phi \in \Phi} |P_\phi| = \binom{m}{2}$.
    The proof is completed by noting that
    \[
      \err(\cH^+_{\Phi_\beta}, \cD) \leq \err(h_\beta, \cD) =\P[h_\beta(X) \neq c^*(X)] \leq \P[\ex] + \P[\neg \ex \wedge (h_\beta(X) \neq c^*(X))],
      \]
    and $\P[\ex] \leq \epsilon$. 

  \end{proof}

 \newcommand{\baseEx}{x_{\mathrm{base}}}
 \newcommand{\baseLab}{y_{\mathrm{base}}}
 \begin{algorithm}[t]
   \caption{Feature discovery protocol for Discriminative Feature Feedback}
   \label{alg:inter}
   \begin{algorithmic}[1]
\INPUT{ Mistake limit $b$, threshold parameter $\beta \in (0,1)$}
\OUTPUT{A feature set $\hat{\Phi} \subseteq \Phi$}
   \State Get initial labeled example $(x_0,y_0)$. 
   \State $t \leftarrow 1$
 \State Initialize a function $F:\Phi \rightarrow \nats$ to constant $0$.
     \While{$\norm{F}_1 \leq b/2$}
 \State $t \leftarrow t+1$
 \State Get next example $x_t$
 \State Predict $y_0$ with explanation $x_0$.
 \State Get true label $y_t$. ~~~\comment{ignore any additional feedback}.
 \State $\baseEx \leftarrow x_t$, $\baseLab \leftarrow y_t$.
\While{\true} ~~~\comment{loop breaks below after getting an incorrect prediction}
\State $t \leftarrow t+1$
 \State Get next example $x_t$
 \State Predict $\baseLab$ with explanation $\baseEx$.
\If {prediction was incorrect}
 \State Get true label $y_t$ and discriminative feature $\phi_t$.\label{line:feedback}
\State $\phi_+ \leftarrow $ the positive feature out of $\phi_t$ and $\neg \phi_t$.
\State $F(\phi_+) = F(\phi_+)+1$.\label{line:fcount}
\State \texttt{break from inner loop.}
\EndIf
\EndWhile
\EndWhile
\State Return $\hat{\Phi} := \{\phi \mid F(\phi) \geq \beta \norm{F}_1\}$.
 \end{algorithmic}
\end{algorithm}

Next, we provide a feature discovery protocol that allows identifying features above a certain threshold. The protocol, listed in \myalgref{inter}, repeatedly sets an observed example and its label as a baseline, and uses them to predict the label of the next examples, until one of these predictions is incorrect and provides feature feedback. The observed features are counted, and the output set $\hat{\Phi}$ includes all the features that have been observed at least a $\beta$ fraction of the time. We denote $\norm{F}_1 := \sum_{\phi\in \Phi} F(\phi)$.

We now show that with a high probability, this protocol identifies all common features . Consider a feature feedback $\phi_t$  that is observed in line \ref{line:feedback}. Note that only a single feature feedback is obtained for each instance of $(\baseEx, \baseLab)$. The distribution of the examples $\baseEx$ used for prediction is equal to $\cD_X$. The distribution of examples $x_t$ that get an incorrect answer, given $\baseEx$, is equal to $\cD_X$ conditioned on $y_t \neq \baseLab$. Thus, for any $\phi \in \Phi$, 
    \begin{align}
      \P[\phi_t = \phi] &= \sum_{G \in \cG} P[\baseEx \in G] \sum_{G':\{G,G'\} \in P_\phi}\P[x_t \in G'\mid y_t \neq \baseLab]\notag\\
      &\geq \sum_{G \in \cG} P[G] \sum_{G': \{G,G'\} \in P_\phi} P[G'] \geq 2\beta_\phi.\label{eq:betaphi}
    \end{align}
    Moreover, the collected $\phi_t$ are statistically independent. The following lemma thus follows.

    \begin{lemma}\label{lem:fsize}
      Suppose that $\norm{F}_1 \geq 6\log(1/(2\delta\beta))/\beta$. Then with a probability at least $1-\delta$, $\Phi_\beta \subseteq \hat{\Phi}$.      
    \end{lemma}
    \begin{proof}

Fix some $\phi \in \Phi$, and let $\hat{p}_\phi := F(\phi)/\norm{F}_1$. By \eqref{betaphi}, $\E[\hat{p}_\phi] \geq 2\beta_\phi$.  
Hence, by Bernstein's inequality, 
with a probability at least $1-2\delta\beta$,
\[
  \hat{p}_\phi \geq 2\beta_\phi - \frac{2\log(1/(2\delta\beta))}{3\norm{F}_1} - \sqrt{\frac{2\beta\log(1/(2\delta\beta))}{\norm{F}_1}}.
  \]
  Since $\norm{F}_1 \geq 6\log(1/(2\delta\beta))/\beta$ and $\beta_\phi \geq \beta$, this implies that $\hat{p}_\phi \geq 2\beta_\phi - \beta \geq \beta$.

Now, note that
      \[
        \sum_{\phi \in \Phi}\beta_\phi = \sum_{\{G,G'\}\subseteq \cG}\P[G]\P[G'] = \half \sum_{G \in \cG}\P[G]\cdot\sum_{G' \in \cG \setminus \{G\}}\P[G'] \leq \half.
      \]
Therefore, there are at most $1/(2\beta)$ features $\phi$ such that $\beta_\phi \geq \beta$.
  By a union bound, this holds simultaneously for all $\phi$ such that $\beta_\phi \geq \beta$, with probability $1-\delta$. 
\end{proof}

We can now conclude a Feature Influence property for our Discriminative Feature Feedback component model.
\begin{thm}\label{thm:infcomp}
  Let $\cH = \cH^+_\Phi$ be the hypothesis class defined above. For any $\beta,\delta, \epsilon \in (0,1)$, 
  the Discriminative Feature Feedback setting with a representation size $m$ has $\sqrt{\beta}\cdot m^2/2$-\emph{Feature Influence} with exception rate $\epsilon$, confidence $\delta$, mistake bound $12\log(1/(2\delta\beta))/\beta$ and capacity bound $m\log(m) + m^2\log(3/\beta)$.
\end{thm}
\begin{proof}
  The feature discovery protocol in \myalgref{inter} makes at most two mistake for every $\phi_t$ it obtains. Thus, it obtains at least $b/2$ independent random samples of features $\phi$. Setting $b = 12\log(1/(2\delta\beta))/\beta$, it follows from \lemref{fsize} that with probability $1-\delta$ we have $ \Phi_\beta \subseteq \hat{\Phi}$, so $\cH^+_{\Phi_\beta} \subseteq \cH^+_{\hat{\Phi}}$. By \lemref{errhb}, $\err(\cH^+_{\hat{\Phi}}, \cD) \leq \err(\cH^+_{\Phi_\beta}, \cD) \leq \epsilon + \sqrt{\beta}\cdot m^2/2$, as required by the definition of Feature Influence.

  To prove the capacity bound, note that $|\hat{\Phi}| \leq 1/\beta$. Therefore, the size of $\cH_{\hat{\Phi}}$ can be bounded by the total number of possible decision lists with up to $m$ rules, up to $m-1$ features in each rule, and up to $m$ different labels, to yield $|\cH_{\hat{\Phi}}| \leq (2/\beta+1)^{m(m-1)}\cdot m^m$. Thus, $\log_2|\cH_{\hat{\Phi}}| \leq m\log(m) + m^2\log(3/\beta)$.
\end{proof}

Combining this result with the general Feature Influence result of \thmref{inf}, we get the following bound for the stochastic setting.
\begin{cor}
  In the stochastic setting, there exists an algorithm that obtains a mistake rate of at most $\epsilon + \alpha$ with probability at least $1-\delta$ for i.i.d.~streams of length at least
  \[
    \tilde{O}\left(\frac{m^4+\log(1/\delta)}{\alpha^3}\right),
  \]
  
\end{cor}

\begin{proof}
  The claim  immediately follows by setting $\beta := \alpha^2/m^4$ in \thmref{infcomp}, which gives an $\alpha/2$-Feature Influence with exception rate $\epsilon$, confidence $\delta$, mistake bound \mbox{$\tilde{O}((m^4+\log(1/\delta)/\alpha^2)$} and capacity bound $\tilde{O}(m^2\log(1/\alpha))$, and applying \thmref{inf}.
\end{proof}

The results shows that using the algorithm proposed above, the mistake rate converges to the exception rate for $n \rightarrow \infty$, as desired.

\section{Conclusions}
In this work, we showed that robust interactive learning with Discriminative Feature Feedback can be achieved with significantly improved mistake bound compared to previous results, both in the adversarial setting and in the stochastic setting. Several interesting open problems remain for future work, such as the existence of an efficient algorithm for the stochastic setting that converges to the exception rate, and the possibility of improving the mistake bound in the general adversarial setting to match the adversarial unique-label setting.

  \subsubsection*{Acknowledgements}

The author would like to thank Sanjoy Dasgupta for helpful discussions. The author thanks Omri Bar Oz for identifying typos in an earlier version of this paper.

\bibliographystyle{abbrvnat}
\bibliography{mybib}

\appendix
  \section{A Parameter-free algorithm}\label{app:parameterfree}
  \newcommand{\tk}{\widetilde{k}}
\newcommand{\tm}{\widetilde{m}}

The algorithm \algknown, presented in \secref{newrobust}, requires a number of components $m$ as input. Its mistake bound depends on $m$ and on $k$, the minimal number of exceptions for a representation with at most $m$ components.
We now show that a nested doubling trick allows applying \algknown\ even if no upper bound $m$ is known. The cost is only a logarithmic factor in the number of mistakes. A similar approach can be applied for the stochastic algorithm described in \secref{stochastic}.
 \begin{algorithm}[t]
  \caption{\algunknown: A parameter-free robust Discriminative Feature Feedback algorithm for the adversarial setting}
  \label{alg:unknown}
  \begin{algorithmic}[1]
    \State $v \leftarrow 1$
    \While{true}
       \State $\tm \leftarrow 1$.
       \While{$\UB(\tm,0) \leq v$}
           \State $\tk \leftarrow 0$
           \While{$\UB(\tm,\tk) \leq v$}
           \If{$\UB(\tm,\tk) >  v/2$}
           \State Run \algknown$(\tm,\tk)$ on the next examples in the stream;\\ \hspace{5.6em} \texttt{Break} from the loop if more than $\UB(\tm,\tk)$ mistakes have occured. 
           \EndIf
             \State $\tk \leftarrow 2\tk+1$
           \EndWhile
           \State $\tm \leftarrow 2\tm$  
       \EndWhile
       \State $v \leftarrow 2v$
    \EndWhile
\end{algorithmic}
\end{algorithm}

The parameter-free algorithm \algunknown\ is provided in \myalgref{unknown}. We now prove \thmref{unknownbound} by bounding the number of mistakes made by  \algunknown\ for all pairs of $m$ and $k$ such that the input stream is consistent with some representation of size $m$ with at most $k$ exceptions.

\begin{proof}[Proof of \thmref{unknownbound}]  The upper bounds of $m$ and $k$ in \thmref{mistakebound} hold also for any sub-sequence of the input sequence. Therefore, once \algunknown\ runs \algknown$(m',k')$ for some $m' \geq m, k' \geq k$, it will never break from the loop. This will occur for some $m' < 2m$, $k' \leq 2k$, when $v = v'$ for some $v' \leq 2\UB(m', k') \leq 2\UB(2m, 2k) \leq 8\UB(m,k)$, where the last inequality follows from the definition of $\UB$.

  In a loop over $\tk$ with given $v$ and $\tm$, the number of rounds is $J \leq \argmax\{ j \geq 0 \mid \UB(\tm, 2^j-1) \leq v\}$. The total number of mistakes made by \algunknown\ in this loop is upper-bounded by

  \begin{align*}
    \sum_{j = 1}^{J} (1 + \UB(\tm, 2^j-1)) &= (J+1)(1+\tm(\tm-1)) + \tm\sum_{j = 1}^{J} (2^j-1) \\
    &\leq (J+1) (1+\tm(\tm-1))+ \tm (2^{J+1}-2).
  \end{align*}
  Let $\hat{k}  = 2^J-1$. By the definition of $J$, $\UB(\tm,\hat{k}) < v$. Thus, the 
  RHS of the inequality above is upper-bounded by
  \begin{align*}
    &(\log_2(\hat{k}+1)+1)(1+\tm(\tm-1)) + 2\tm\hat{k} \leq (\log_2(\hat{k})+2)\cdot (\UB(\tm,\hat{k})+1)
    \\
    &\leq (\log_2(\hat{k})+2)\cdot (v+1),
  \end{align*}

  where the penultimate inequality follows from the definition of $\UB$.
  To upper bound $\hat{k}$, note that since $\UB(\tm, \hat{k}) \leq v$,
by the definition of $\UB$, $\hat{k} \leq v$. 
Therefore, for given $v,\tm$, the total number of mistakes is at most

  $(\log_2(v)+2)\cdot (v+1)$.
 
  From the definition of $\UB$, we have that for a given value of $v$, the maximal value of $\tm$ that \algunknown\ might try satisfies $\tm \leq v$, since $\UB(\tm, 0) = \tm(\tm-1) \leq v$. Therefore, the number of rounds over $\tm$ for a given $v$ is at most $\log_2(v)+1$. It follows that the total number of mistakes for loop with a given $v$ is at most

    $(\log_2(v)+1)(\log_2(v)+2)\cdot (v+1)$.

Summing over all values of $v$ that \algunknown\ tries until reaching $v'$,
we get a total mistake bound of
\[
\sum_{i = 0}^{\log_2(v')} (i+1)(i+2)(2^i+1) \leq 4v'\log^2(v') \leq 32\UB(m,k)\log^2(8\UB(m,k)),
\]
as claimed.
\end{proof}

\end{document}